\documentclass{article} 
\usepackage[margin=1in]{geometry}
\usepackage{graphicx} 
\usepackage{mathnotation}
\usepackage{algorithmic}
\usepackage{subcaption}
\usepackage{cleveref}
\usepackage{xcolor}
\usepackage{enumitem}
\usepackage{wrapfig}
\usepackage{natbib}
\newcommand{\spire}{\texttt{SPIRE}\xspace}
\newcommand{\kaan}[1]{\textcolor{black}{#1}}
\newcommand{\rem}[1]{\textcolor{black}{#1}}

\usepackage[parfill]{parskip}

\title{\spire: Conditional Personalization for Federated Diffusion Generative Models}

\author{Kaan Ozkara\thanks{Department of Electrical and Computer Engineering, University of California Los Angeles\ \{kaan@g.ucla.edu, ruida@g.ucla.edu, suhas@ee.ucla.edu\}} 
\and 
        Ruida Zhou$^*$
\and 
        Suhas Diggavi$^*$
}

\begin{document}

\maketitle

\begin{abstract}
Recent advances in diffusion models have revolutionized generative AI, but their sheer size makes on‑device personalization, and thus effective federated learning (FL), infeasible. We propose Shared Backbone Personal Identity Representation Embeddings (\spire), a framework that casts per‑client diffusion based generation as conditional generation in FL. \spire factorizes the network into (i) a high‑capacity global backbone that learns a population‑level score function and (ii) lightweight, learnable client embeddings that encode local data statistics. This separation enables parameter‑efficient fine‑tuning that touches $<0.01\%$ of weights. We provide the first theoretical bridge between conditional diffusion training and maximum‑likelihood estimation in Gaussian‑mixture models. For a two‑component mixture we prove that gradient descent on the DDPM with respect to mixing weights loss recovers the optimal mixing weights and enjoys dimension‑free error bounds. Our analysis also hints at how client embeddings act as biases that steer a shared score network toward personalized distributions. Empirically, \spire matches or surpasses strong baselines during collaborative pre‑training, and vastly outperforms them when adapting to unseen/new clients, reducing Kernel Inception Distance while updating only hundreds of parameters. \spire further mitigates catastrophic forgetting and remains robust across fine-tuning learning‑rate and epoch choices.

\end{abstract}

\section{Introduction}\label{sec:intro}

Modern diffusion models have achieved striking success in image, audio and video synthesis, yet their ever‑growing scale puts them beyond the reach of the low‑power devices that generate most of today’s data.  \emph{Federated learning} (FL) offers a remedy by training a single global model across millions of clients without centralizing raw data.  Unfortunately, a one‑size‑fits‑all generator rarely serves individual users well: camera characteristics, usage patterns and personal preferences all induce \emph{client‑specific} distributions that differ markedly from the global average.  These discrepancies are most pronounced for \emph{new clients}, who arrive after the global model has been trained and possess only a handful of local samples. A key question is, how can we enable a large diffusion model with the flexibility to personalize quickly—even when users' data are scarce and heterogeneous?

\textbf{Our key insight.}  Many real‑world data sets share a \emph{high‑dimensional structure} that is common across clients (e.g., the geometry of natural images) plus a \emph{low‑dimensional} client‑specific component (e.g., class priors, color statistics or individual identities).  We formalize this intuition through a simple heterogeneity model in which shared parameters~$\phi$ captures the complex structure, whereas each client controls lightweight parameters~$\gamma_j$. For diffusion models we operate the split by (i)~training a \emph{shared backbone} that learns a population score function and (ii)~injecting a \emph{conditioning embedding} that acts as a signature of client statistics.  Personalization thus reduces to jointly learning $\phi$ from all users and individually learning \(\gamma_j\)—a few hundred parameters instead of millions—making it feasible to adapt on‑device. Our key contributions can be summarized as follows.
\vspace{-6pt}
\setlist[itemize]{noitemsep}
\begin{itemize}[leftmargin=1pt]
    \item We introduce \spire, a framework that treats per‑client diffusion modeling as a conditional generation problem in FL.  \spire decouples global and local learning through an embedding‑based conditioning mechanism, enabling parameter‑efficient fine‑tuning for new clients. who have not participated in the federated pretraining. 
    \item We provide the \emph{first theoretical analysis} connecting conditional diffusion training to maximum‑likelihood estimation in Gaussian‑mixture models (GMMs).  In the two‑component case we show that gradient descent on the DDPM objective with respect to mixing weights (which corresponds to low dimensional client specific parameters $\gamma_j$) finds the global maximum and derive dimension‑free error bounds (Theorem~\ref{thm:mixing-weight}). The GMM also provides a theoretical insight into conditioning mechanism in our architecture, including bias to steer personalization.
    \item We develop a practical FL algorithm that trains backbone and embeddings jointly, requires \emph{no additional communication} during personalization, and is robust to catastrophic forgetting. Our experiments on \textsc{MNIST}, \textsc{CIFAR‑10} and \textsc{CelebA} \cite{liu2015faceattributes} demonstrate that \spire matches or exceeds strong baselines during collaborative pre‑training and \emph{vastly outperforms} them when adapting to new clients while updating fewer than $0.01\%$ of the parameters.
\end{itemize}
\vspace{-5pt}

\textbf{Personalized FL. } There has been very limited work on personalized federated diffusion models, despite the vast amount of work for personalized FL, particularly for supervised learning scenarios. The previous personalized FL works have focused on adjusting the optimization criteria \citep{dinh2020personalized,hanzely2020federated, ozkara2023}, meta learning based approaches \citep{fallah2020personalized,acar2021debiasing, kohdak2019adaptive}, clustering based approaches \cite{mansour2020approaches}, knowledge distillation based approaches \citep{ozkara2021quped, lin2020ensemble} and so on. Conceptually, even though exclusively supervised learning is considered, the most related to our work is shared representation approach in \citep{collins-icml21}, where the authors vertically partition neural networks as global feature extractors and local heads. In contrast, our approach is a horizontal partitioning, where high dimensional backbone is trained with light weight conditioning information based on client's data statistics. In terms of the task, \citep{ozkara2025adept} also looks at personalized diffusion models, through a lens of statistical hierarchical Bayesian framework. Their method causes a regularization in the loss function, whereas, ours is based on a conditional architecture. Furthermore their method requires two separate models which is less suitable for new clients. \\
\textbf{Conditional diffusion models.} Conditioning is one of the key properties to add additional information while training diffusion models \cite{dhariwal2021diffusion}. The most common way of conditioning is to embed the information to a vector space and use in intermediate layers as bias (or modulate the activations). Our work explains this use in a principled way and connects it to a theoretical view (see Section~\ref{sec:theory}), for the first time, as far as we are aware. We utilize the embeddings to estimate a function that extracts the client data statistics. Furthermore, our work utilizes conditioning mechanism as a way to do PEFT, we are not aware of any other works that consider conditioning for FL and new clients. \\
\textbf{Diffusion models for GMMs.} The \emph{two‑component symmetric Gaussian‑mixture model (GMM)} has long served as a testbed for analyzing EM‑algorithm convergence (see, e.g., \citep{daskalakis17b,kwon2020emalgorithmgivessampleoptimality}). \citep{shah2023learning} shows that diffusion models (DDPM loss) can be learned through gradient descent to estimate GMM score function when mixing weights, covariances are known. \citep{chen2024learninggeneralgaussianmixtures}, further shows diffusion models can be used to efficiently learn general GMMs, through piecewise polynomial function approximations. In Section~\ref{sec:theory}, our method extends the former; in a distributed and heterogeneous setup, we use gradient descent to learn the mixing weights of the GMMs and prove dimension-free bounds. 

Beyond the studies highlighted above, the literature on Personalized Federated Learning (FL) has evolved along several complementary tracks. Meta‑learning approaches aim to learn initialization points that quickly adapt to each client’s data \citep{fallah2020personalized,acar2021debiasing,kohdak2019adaptive}. Regularization‑based methods encourage personalization by adding client‑specific penalties to the global objective \citep{deng2020adaptive,mansour2020approaches,hanzely2020federated}. Clustered FL partitions clients into groups with similar data distributions \citep{zhang2021personalized,mansour2020approaches,ghosh2020efficient,marfoq2021federated}, while knowledge‑distillation techniques transfer information between local and global models \citep{li2020federated,ozkara2021quped}. Multi‑task formulations treat each client as a related task to optimize jointly \citep{dinh2020personalized,smith2017federated,vanhaesebrouck2017decentralized,zantedeschi2020fully}, and common‑representation methods share feature extractors across clients \citep{du2021fewshot,tian2020rethinking,collins-icml21}. A hierarchical Bayesian viewpoint has recently inspired new supervised FL algorithms with adaptation capabilities \citep{ozkara2023,chen2022selfaware,kotelevskii2022fedpop}.

\section{Problem setup: Personalized federated diffusion generative models}

\begin{figure}[h]
    \centering
    \begin{subfigure}[b]{0.45\linewidth}
        \centering
        \includegraphics[width=\linewidth]{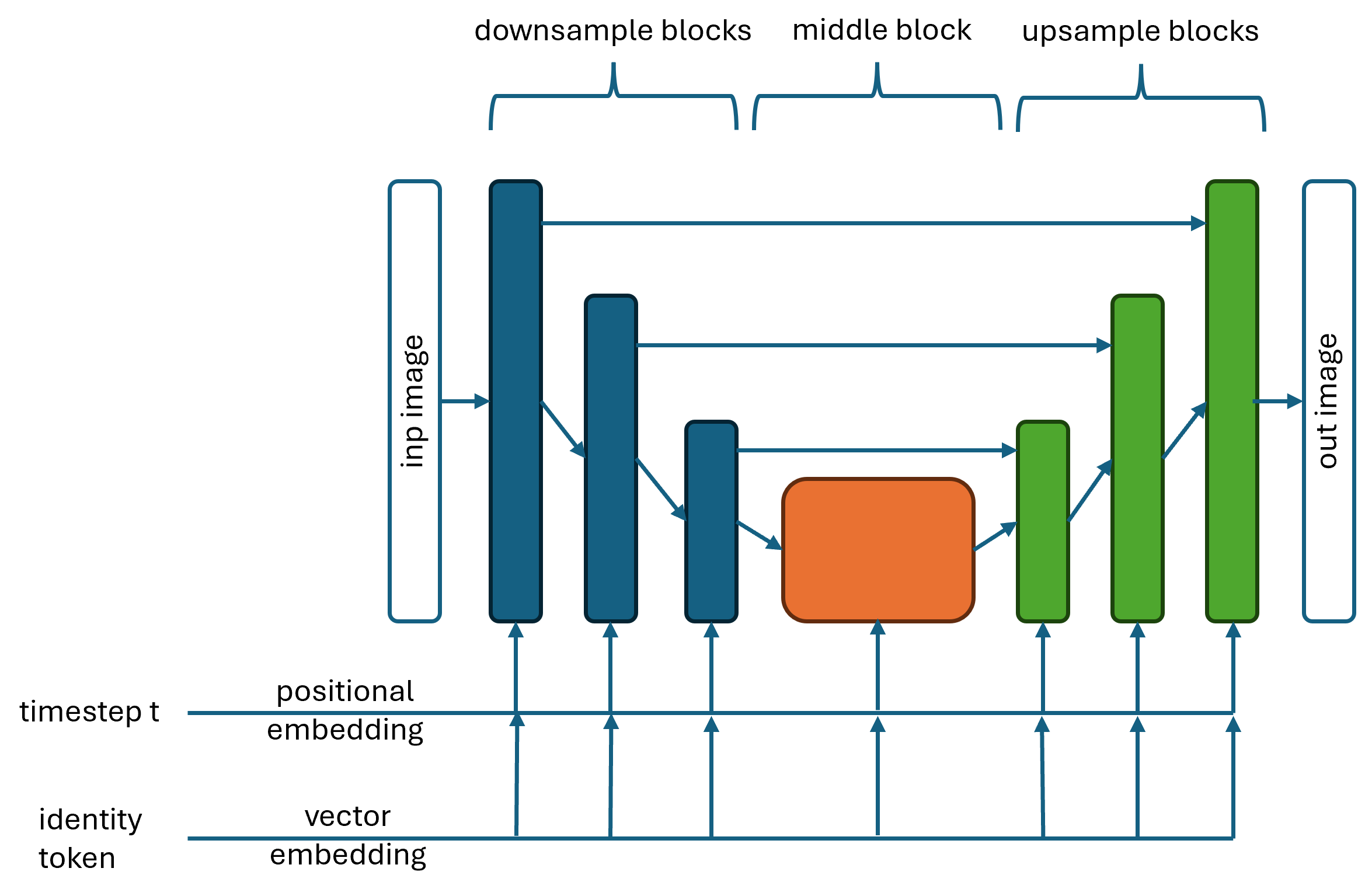}
        \subcaption{Overall diffusion model architecture.}
        \label{fig:architecture}
    \end{subfigure}
    \hfill
    \begin{subfigure}[b]{0.45\linewidth}
        \centering
        \includegraphics[width=\linewidth]{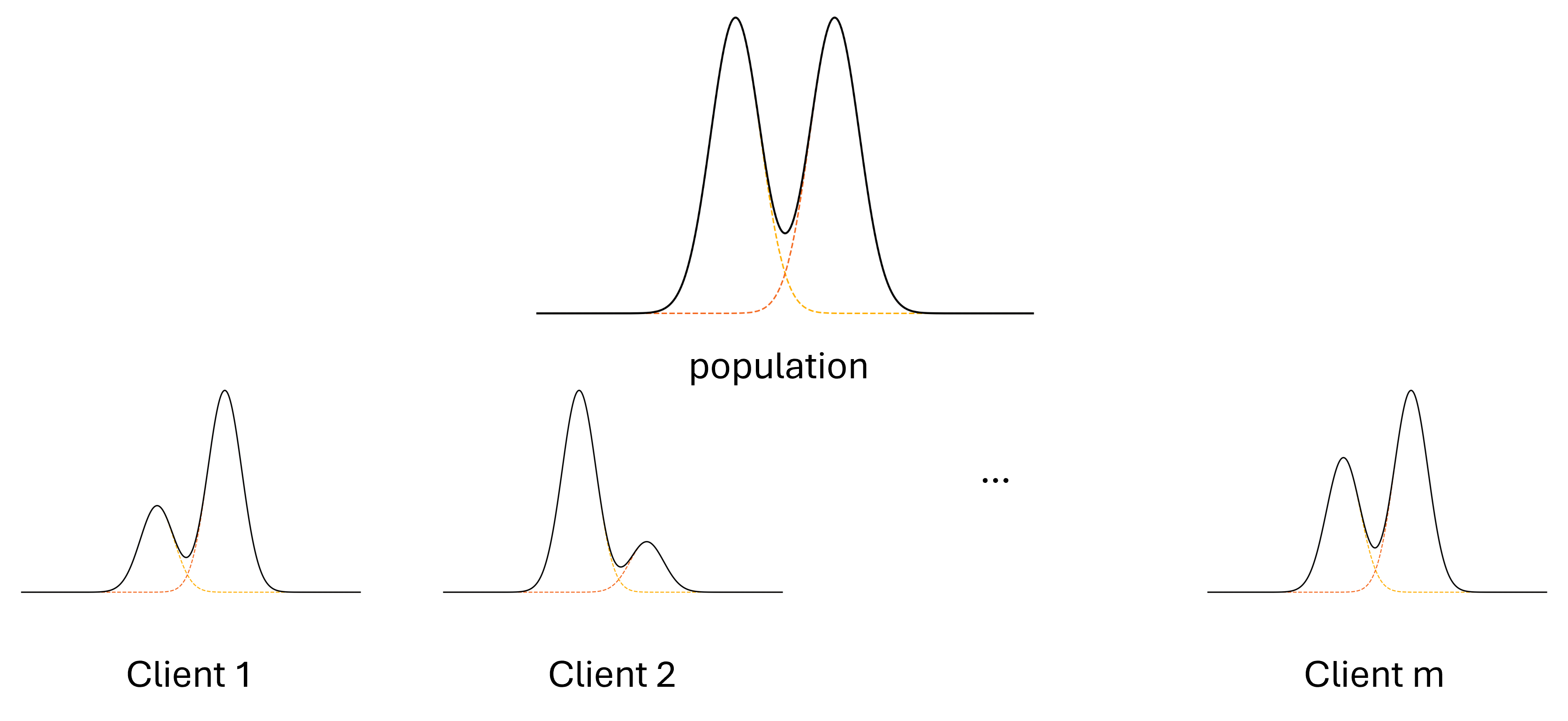}
        \subcaption{Model of heterogeneity for GMMs.}
        \label{fig:population-model}
    \end{subfigure}
    \caption{(Left) Each client holds an identity token which is mapped to a vector embedding that is trained throughout with the backbone. During the inference, a client can use its identity token to do a personalized generation. (Right) In GMM heterogeneity model, all clients share the same means but the mixing weights are individual.  }
    \label{fig:synth-figures}
\end{figure}


In this section, we first introduce a heterogeneity model of clients' data that is partitioned into two: a complex common structure and a lower dimensional individual structure that is specific to clients. A concrete example is a mixture model where high dimensional cluster centers are common for every client but mixing weights are individual (e.g. see Figure\ref{fig:population-model}). Later, we will explain (conditional) diffusion models for personalization and the overall federated learning setup that we consider. Lastly, we will explain how the formulation is especially useful for new clients.  
\subsection{Problem setup and background}
We consider a conditional personalization approach for learning individual diffusion models tailored to local data from clients. Let us denote $q^{(j)}(x)$ as the data distribution of a particular client $j$. We assume $q^{(j)}(x)$ is implicitly conditioned on two types of variables, i.e. $q^{(j)}(x)=q^{(j)}(x|\phi, \gamma_j)$ where $\phi \in \Phi$ is a high dimensional shared parameter, and $\gamma_j \in \Gamma$ is a lower dimensional parameter that is specific to a client $j$. 
Note that through the Bayes rule, the pdf implies a mapping from data distribution of a client and the shared parameters to $\gamma_j$, there is a (random/data dependent) mapping such that $\gamma_j = g(x,\phi), x\sim q^{(j)}$. To motivate, we give a theoretical and an empirical example.

\begin{example}
Consider data generated by a two component GMM with mean, mixing weights, and covariance parameters $(\pm \mu, \{w_j\}_{j=1}^m,\Sigma)$, each client shares the high dimensional mean parameter $\mu$ but individually possesses different mixing weights $w_j$ of the components. In this example $\phi, \gamma_j$ corresponds to $\mu,w_j$ respectively.
\end{example}

\begin{example}
    Consider a non-parametric score function that is approximated through a neural network with parameters $\vct \theta_j$, for a client $j$. Potentially, $\vct \theta_j$ can be partitioned into high dimensional $\vct \theta_{bb}$ that is learned collaboratively and low dimensional $\vct e_j$ that is learned locally. The goal is to use $\vct \theta_{bb},\vct e_j$ as proxies for $\phi,\gamma_j$ respectively. Ideally, $\vct \theta_{bb}$ should describe a general structure in generation task, and $\vct e_j$ should add personalization according to input data statistics.
\end{example}
The goal of the personalized diffusion models is to learn $q^{(j)}(x)$ that is the individual conditional data distribution of a client given $\phi,\gamma_j$. In federated learning, every client holds a small dataset $\{X_i^{(j)}\}_{i=1}^n\sim q^{(j)}$ (assuming same $n$ for simplicity), yet the globally data is from $m$ participating clients, where $m$ is assumed to be larger tan $n$. This abundance enables collaborative learning of the complex $\phi$, while each client estimates its low‑dimensional $\gamma_j$ locally. New clients, absent from pre‑training, can fine‑tune only $\gamma_j$ with few samples, achieving sample‑efficient personalization. 


\textbf{Diffusion Models. } Conventionally, diffusion models are composed of several parts. First, a \textit{forward process} that mathematically describes the gradual transformation from a true sample to noise. A corresponding \textit{backward process} is necessary to formalize how the score function can be used to denoise a noisy sample to obtain a data point from the true distribution. Lastly, since the true score function is unknown one needs to formulate an optimization problem (i.e. \textit{score matching}) to estimate it to be used in the backward process. Due to space limitations we define the original diffusion model in Appendix~\ref{app:proofs}. Here we introduce the personalized diffusion models. In the personalized FL setup, we would like to generate samples from clients' individual probability distributions. We first start with the forward process. The forward diffusion process is defined as
\begin{align} \label{eq:forward-pers}
    \mathrm{d} X^{(j)}_t = -\delta_t X^{(j)}_t \, \mathrm{d}t + \sqrt{2 \delta_t} \, \mathrm{d} {W}_t, \quad X^{(j)}_0 \sim q^{(j)}_0(x).
\end{align}
here $q^{(j)}_0(x)$ denotes the true underlying distribution that generates client $j$s data, similarly $q^{(j)}_t(x)$ will be used to denote a latent distribution, $\delta_t$ is a drift coefficient, commonly assumed a constant e.g. 1. ${W}_t$ is the standard Brownian motion. $X^{(j)}_t \sim q^{(j)}_t$ is a latent variable. Equivalently, it can be shown \citep{shah2023learning} for some $\beta_t:\lim_{t \rightarrow\infty}\beta_t = 1, \beta_0=0$,
\begin{align}\label{eq:equ-fwp}
    X^{(j)}_t = \sqrt{1-\beta_t^2}X^{(j)}_0 + \beta_t Z_t, \text{with } X^{(j)}_0 \sim q^{(j)}_0 \text{ and } Z_t \sim \mathcal{N}(0,\mathrm{I}).
\end{align}
Corresponding to the forward process, is the backward SDE process, which is defined to formalize recovering a data sample from true distribution, given a corrupted sample and score function, 
\begin{align}\label{eq:backward}
    \mathrm{d} \overleftarrow{X}_{t}^{(j)} = 
\left[ 
\delta_{T-t} \overleftarrow{X}^{(j)}_{t} + 2 \delta_{T-t} \nabla \log q^{(j)}_{T-t}(\overleftarrow{X}^{(j)}_{t}) 
\right] \mathrm{d}t 
+ \sqrt{2 \delta_{T-t}} \, \mathrm{d} {W}_t, 
\quad \overleftarrow{X}^{(j)}_{0} \sim q^{(j)}_T(\cdot).
\end{align}
In the backward process, we use $\overleftarrow{X}^{(j)}_{0}$ to denote a sample from pure noise $q^{(J)}_T$. $\nabla \log q^{(J)}_{T-t}(\overleftarrow{X}^{(j)}_{t})$ is the true score function at time $T-t$ (or at time $t$ in terms of the forward process). Of course the true score function is unknown, hence we need to estimate it from the data using a fixed pure noise distribution. To that end, one constructs the score matching optimization criterion,
\begin{align}
\min_{\theta} \int_\tau^T 
\mathbb{E}_{X^{(j)}_t \sim q^{(j)}_t(\cdot)} 
\left[
\left\|
s_{\theta_j}(x_t, t) - \nabla \log q^{(j)}_t( X^{(j)}_0)
\right\|^2
\right] 
\, \mathrm{d}t,
\end{align}
where $\theta_j$ is a neural network that is assumed to parameterize the score function. Equivalent to score matching is the DDPM view, which yields \cite{chen2023sampling},
\begin{align}
\min_{\theta} \int_\tau^T 
\mathbb{E}_{X^{(j)}_0 \sim q^{(j)}_0(\cdot), Z_t} 
\left[
\left\|
s_{\theta_j}(X^{(j)}_t, t) + \frac{Z_t}{\beta_t}
\right\|^2
\right] 
\, \mathrm{d}t.
\end{align}

\section{\spire: Personalized FL using conditional diffusion models}

 
In our work, we consider using client's data (through learning its embedding) as the conditioning information to the model. This results in an horizontal split of the model. During pre-training, $\vct \theta_{bb}, \{\vct e_j\}_j$ are trained jointly, hence, the model implicitly learns to estimate the function $\vct e _j = g^{(j)}(x^{(j)}, \vct \theta_{bb})$. One can see that finding the embedding in the diffusion model training  corresponds to extracting data statistics, in other words, $g^{(j)}(x,\theta_{bb}) = \arg \min_{e_j} L_{ddpm}(q^{(j)},\theta_{j}) $, where $L_{ddpm}$ is the score matching loss in \eqref{eq:empicial_score_matching} and $x^{(j)} \sim q^{(j)}(x)$. 

\textbf{Form of $\vct \theta_j$. } Our heterogeneity modeling yields separability of $\vct \theta_j$ into two parts. A more complicated (higher dimensional) backbone part $\vct \theta_{bb}$ that is shared among estimated scores of different clients, and a lower dimensional embedding of clients' data statistics, $\vct e_j$, that is individual to each client. We can write $\vct \theta_j=[\vct \theta_{bb},\vct e_j]$. During training, overall personalized federated DDPM objective consists of finding $m$ different models (with shared backbone and individual identity embedding in our case), 
\begin{align}
\label{eq:score_matching}
    \min_{\{\theta_j\}_{j=1}^m} \frac{1}{m} \sum_{j=1}^{m} \int_\tau^T 
\mathbb{E}_{X_0 \sim q^{(j)}_0(\cdot), Z_t} 
[
\|
s_{\theta_j}(X^{(j)}_t, t) + \frac{Z_t}{\beta_t}
\|^2
] 
\, \mathrm{d}t.
\end{align}

Empirical version of the DDPM loss can be defined through averaging over local dataset. 
\begin{align} \label{eq:empicial_score_matching}
    \min_{\{\theta_j\}_{j=1}^m} \frac{1}{m}\sum_{j=1}^{m} L^{(j)}_{ddpm}(\{X^{(j)}_i\}_{i=1}^n,\theta_j):=  \frac{1}{n} \sum_{j=1}^{n}\int_\tau^T 
\mathbb{E}_{Z_t} 
\left[
\left\|
s_{\theta_j}((X_i^{(j)})_t, t) + \frac{Z_t}{\beta_t}
\right\|^2
\right] 
\, \mathrm{d}t.
\end{align}
In practice instead of taking expectation over timesteps, we sample randomly. A client, then, can plug in the estimated score ($s_{\theta_j}$) into the backward process \eqref{eq:backward} to generate personalized samples.


\subsection{Federated training setup. }

During training, each client maintains a personalized embedding (through a identity token, e.g., client id number) which acts as a key for the backbone model. The embeddings are kept in client and trained locally, and backbone model is shared and trained by essentially FedAvg. Algorithm\ref{algo:training} summarizes the key steps. 

\rem{
\textbf{Personalization to new clients. } Personalization for new clients is one of the key challenges in federated learning. It is especially, pronounced for generative models due to scale of the models. A new client joining the federated ecosystem, later than pre-training stage, can utilize the pre-trained $\vct \theta_{bb}$ and train only its own embedding parameters. In a way, this defines a natural way of parameter efficient fine-tuning (PEFT) for personalization, where we are estimating $q^{(new)}$ through $(\vct \theta_{bb}, \vct e_{new})$. See Algorithm~\ref{algo:new-clients}, for new client fine-tuning details.  
}
\begin{algorithm}[h]
\caption{Conditional collaborative personalized pre-training for \spire}
{\bf Input:} Number of iterations $K$, learning rate ($\eta$), batch size $b$, number of local iterations $\tau$\\
\begin{algorithmic}[1]	 \label{algo:training}
		\STATE \textbf{Initialize} backbone model $\vct{\theta}_{bb}$, and client embeddings $\{\vct{e}_j\}_{j=1}^m$, set $\vct\theta_j = [\vct\theta_0,\vct{e}_i]$.
        \FOR{$k=1$ \textbf{to} $K$}
        \FOR {$i=1$ \textbf{to} $m$} 
		\STATE  \textbf{if} {$\tau$ divides $k$ \textbf{then}:} Receive $\vct \theta_{bb,k}$ from server and set $\vct\theta_{j,k} = [\vct\theta_{bb,k},\vct{e}_{j,k}]$
            \STATE Sample $(\{X_i\}_{i=1}^b,\{t_i\}_{i=1}^b) $, use \eqref{eq:equ-fwp} to obtain noisy samples $\{X_{t_i}\}_{i=1}^b$
            \STATE Take gradient step on local DDPM loss $\vct\theta_{j,k+1} = \vct\theta_{j,k} - \nabla_{\theta_j} L_{ddpm}^{(j)}(\{X_{t_i}\}_{i=1}^b)$
            \STATE \textbf{if} {$\tau$ divides $t+1$ \textbf{then:}} extract updated $\theta_{j,bb,k+1}$ from $\theta_{j,k+1}$ and send to server
        \ENDFOR
        \STATE On server aggregate $\theta_{bb,k} = \frac{1}{m} \sum_{j=1}^m \theta_{j,bb,k}$, and broadcast to clients
        \ENDFOR
\end{algorithmic}
\end{algorithm}

\textbf{Algorithm~\ref{algo:training} }  
Line~1 initializes the shared backbone parameters $\boldsymbol{\theta}_{\mathrm{bb}}$ and the per‑client embeddings $\{\mathbf{e}_j\}$. Whenever the synchronization test $\tau \mid k$ succeeds, each client receives the latest backbone (line~5). During each local step the client samples a mini‑batch and timesteps, and corrupts the data with the forward process (line~6), and performs a DDPM gradient update on the full parameter vector (line~7). At the next synchronization point, sends the updated backbone slice to the server (line~8), and keeps $\mathbf{e}_{j,k+1}$ locally. The server averages the received slices (line~10) to produce $\boldsymbol{\theta}_{\mathrm{bb},k+1}$, which is broadcast at the start of the next round.

\begin{algorithm}[h]
\caption{Conditional personalized fine-tuning for new clients}
{\bf Input:} Number of iterations $K$, learning rate ($\eta$)\\
\begin{algorithmic}[1]	 \label{algo:new-clients}
		\STATE \textbf{Download} backbone model, initialize the embedding $e_{j,0}$
        \STATE Set $\theta_{j,0}=[\theta_{bb},e_{j,0}]$
        \FOR {$k=1$ \textbf{to} $K$}
            \STATE Sample a batch of samples $\{X_i\}_{i=1}^b$ and diffusion timesteps $\{t_i\}_{i=1}^b$
            \STATE Use forward process \eqref{eq:equ-fwp} to obtain a batch of noisy samples $\{X_{t_i}\}_{i=1}^b$
            \STATE Take gradient step for embedding $\vct e_{j,k+1} = \vct e_{j,k} - \nabla_{\vct e_j} L_{ddpm}^{(j)}(\{X_{t_i}\}_{i=1}^b)$
        \ENDFOR
\end{algorithmic}
\end{algorithm}

\textbf{Algorithm~\ref{algo:new-clients}  }
A new client downloads the current backbone (line~1), forms $\boldsymbol{\theta}_{j,0}$ by appending a freshly initialized embedding (line~2), and for $K$ iterations (line~3) repeats: sampling data and timesteps (lines~4--5) and updating \emph{only} its embedding via the DDPM loss (line~6).  
No further communication is required after the initial download, so personalization remains entirely on‑device.



\section{Theoretical justification: connection to learning mixing weights of GMMs}\label{sec:theory}

 
 Recent work has revisited the same model with \emph{fixed} mixing weights to probe the theoretical foundations of diffusion models and score matching \cite{shah2023learning}.  In our setting we change the emphasis to a distributed setup. The mean vector $\boldsymbol{\mu}\in\mathbb{R}^d$ is shared across all clients, while each client $j$ retains its own mixing weight $w_j\in(0,1)$. This maps directly onto our heterogeneity framework, where a global parameter (the shared mean) co‑exists with lightweight, client‑specific parameters (the scalar weights). \kaan{Conceptually, it reduces to a one‑layer version of our personalized architecture.} Detailed proofs of the results in this section is deferred to Appendix~\ref{app:proofs}.

\subsection{Overview}

We focus on a symmetric GMM with two equal components \textit{in the population level} and unknown mixing weights \textit{in client level} as in Figure~\ref{fig:population-model}. In the personalization setup our model consists of shared means $\pm \mu\in\mathbb{R}^d$, and individual mixing weights of the components $\{(w_j,1-w_j)\}_{j=1}^m$. In particular, the true data distribution for a particular client $j$ is,
\begin{align}
    q^{(j)} = w_j \mathcal{N}(\mu,\text{I}) + (1-w_j)\mathcal{N}(-\mu,\text{I})
\end{align}
where $q^{(j)}$ is the GMM distribution, $\text{I}$ is the identity matrix of size $d \times d$, $\mu \in \mathbb{R}^d$ is the true mean parameter, $w_j \in \mathbb{R}$ is the mixing weight for the positive component for client $j$. This particular mixture distribution is studied in literature, mostly when mixing weights are known and equal, i.e. $1/2$ \citep{shah2023learning}. Below, we consider the setup for a particular client and omit the subscript in $w_j$.

\textbf{Noisy samples.} Utilizing the same OU forward process as in \cite{shah2023learning} and above in \eqref{eq:forward}, it can be shown (see Appendix~\ref{app:proofs}) for $X_t \sim q_t$ we have,
$
    X_t = \exp{(-t)}X_0 + \sqrt{1-\exp{(-2t)}}Z_t, \text{with } X_0 \sim q_0 \text{ and } Z_t \sim \mathcal{N}(0,\text{I}).
$ In the analysis, we will fix a particular timestep $t$. Accordingly, we will focus on the equivalent DDPM objective:
\begin{align}
    \min_{s_t} \mathbb{E}_{X_0, Z_t} \left[\Big\|s_t(X_t) + \frac{Z_t}{\sqrt{1-e^{-2t}}} \Big\|^2\right].
\end{align}

Recall, in our case of two component mixture we have,
$
    q_0 = w \mathcal{N}(\mu,\text{I}) + (1-w) \mathcal{N}(-\mu,\text{I}).
$ Accordingly, $q_t$ can be easily obtained to be (see Appendix~\ref{app:proofs})
\begin{align*}
    q_t = w \mathcal{N}(\mu_t,\text{I}) + (1-w) \mathcal{N}(-\mu_t,\text{I}), \text{ where } \mu_t := \mu \exp{(-t)}.
\end{align*}
The score function has a closed form solution as $\nabla_x \log q_t(x) = r_1(x) \mu_t + r_2(x)(-\mu_t) - x$ \text{ where } $r_1(x)=\frac{w\exp{{(-\|x-\mu_t\|^2}/{2})}}{w\exp{{(-\|x-\mu_t\|^2}/{2})} + (1-w)\exp{{(-\|x+\mu_t\|^2}/{2})}}$, and $r_2(x)= 1-r_1(x)$ are conditional probability of being assigned to a cluster, i.e. 'beliefs'. Accordingly we can obtain the closed form of the score function as outlined in Lemma~\ref{lem:score} (see Appendix~\ref{app:proofs} for the proof).
\begin{lemma}\label{lem:score}
    The score function at diffusion time step $t$ for $q_t$ is defined as follows ,
\begin{align} \label{eq:network}
    \nabla_x \log q_t(x) = \tanh{(\mu_t^\top x-\frac{1}{2}\log({w}/{(1-w)}))}\mu_t-x.
\end{align}
\end{lemma}

\textbf{Connection to conditional architecture. }\kaan{The expression in Lemma~\ref{lem:score} can be viewed as a \emph{single‑layer neural network with a residual (identity) connection}.  
The shared mean vector $\vct\mu_t$ serves as the layer’s weight vector; the log‑odds term $\tfrac12\log\!\bigl(\tfrac{w}{1-w}\bigr)$ is an explicit bias that shifts the pre‑activation $\vct\mu_t^\top x$. After applying the non‑linearity $\tanh(\cdot)$, the network outputs a vector aligned with $\vct\mu_t$, which is then combined with the skip path $-x$.} This construction is exactly the one‑layer special case of the conditional architecture introduced in the previous section: the conditioning signal is carried solely by the client‑specific mixing weight $w$, while the weights $\vct\mu_t$ remain global. Hence the lemma illustrates, in its simplest form, how general conditional diffusion models encode per‑client information through mixing‑weight‑induced biases. Training a conditional model can therefore be seen as \emph{finding the appropriate mixing weights} that encode each condition, leaving the heavy shared parameters untouched. 

Consequently, we consider \eqref{eq:network} as the network architecture ($s_t(X_t)$ parameterized by $\mu_t,w$) to estimate the score function. In the personalized setup, we are interested in a fine-tuning stage where the means are known and each client learns the individual mixing weight. As a result the problem of interest can be rewritten as,
\begin{align} \label{eq:score-loss}
    \min_{w} \mathbb{E}_{X_0\sim q^{(j)}(x), Z_t} \left[\Big\|\tanh{(\mu_t^\top X_t-\frac{1}{2}\log({(1-w)}/{w}))}\mu_t-X_t + \frac{Z_t}{\sqrt{1-e^{-2t}}} \Big\|^2\right].
\end{align}

Our goal is to examine the MSE at the convergence point; to that end, we relate the first order convergence condition with EM updates \cite{kwon2020emalgorithmgivessampleoptimality}, which are provided in Appendix~\ref{app:proofs}.
\subsection{Theoretical results}
Here, first, we discuss the resulting bounds in the two component GMM problem, for learning the mixing weights. Then we look at score estimation error in a two step procedure where the mean is estimated globally, and then each client learns the mixing weight locally. 
\begin{theorem} \label{thm:mixing-weight}
    Fixing a $\mu_t$, MSE of estimating $w$ by a gradient method on the DDPM loss results in,
    \begin{align*}
        \mathbb{E}_{X_0}(w-\hat{w})^2 \leq \frac{w(1-w)}{n}+ \frac{d}{4 \|\mu_t\|^2n}
    \end{align*}
\end{theorem}

\begin{proof}[Proof of Theorem \ref{thm:mixing-weight}]
    We start by taking the derivative of \eqref{eq:score-loss} with respect to $w$ at a particular timepoint $t$. Let $u_t = \mu_t^\top X_t-\frac{1}{2}\log({w}/{(1-w)}), f(X_t) = \tanh{(u_t)}-X_t + \frac{Z_t}{\sqrt{1-e^{-2t}}}$, by the chain rule,

    \begin{align*}
        \frac{\text{d}L}{\text{d}w} &= \frac{\text{d}L}{\text{d}f(X_t)} \cdot \frac{\text{d}f(X_t)}{\text{d}\tanh(u_t)} \cdot \frac{\text{d}\tanh(u_t)}{\text{d}u_t}\cdot\frac{\text{d}u_t}{\text{d}(\frac{1}{2}\log ((1-w)/w))}\cdot\frac{\text{d}(\frac{1}{2}\log ((1-w)/w))}{w} \\
        & = \mathbb{E} \left[2 \mu_t^\top(\tanh(u_t)\mu_t-X_t + Z_t/(\sqrt{1-e^{-2t}}) \tanh'{(u_t)} \frac{-1}{2w(1-w)} \right]\\
        & = \frac{-1}{w(1-w)}\mathbb{E} \left[  \mu_t^\top ( \tanh(u_t)\mu_t - X_t + Z_t/\sqrt{1-e^{-2t}}) \cdot \tanh'(u_t)\right] \\
        & \frac{-1}{w(1-w)}\mathbb{E} \left[  \mu_t^\top( \tanh(u_t)\mu_t - X_t + Z_t/\sqrt{1-e^{-2t}}) \cdot \tanh'(u_t)\right] \\
        & \frac{-1}{w(1-w)}\mathbb{E} \left[  \|\mu_t\|^2\tanh(u_t)\tanh'(u_t) - \mu_t^\top X_t\tanh'(u_t) + \mu_t^\top Z_t/\sqrt{1-e^{-2t}} \tanh'(u_t)\right] \\
        & \stackrel{}{=} \frac{-1}{w(1-w)}\mathbb{E} \left[  \|\mu_t\|^2\tanh(u_t)\tanh'(u_t) - \mu_t^\top X_t\tanh'(u_t) + \|\mu_t\|^2\tanh''(u_t)\right] \\
        & \stackrel{}{=} \frac{-1}{w(1-w)}\mathbb{E} \left[  -\|\mu_t\|^2\tanh(u_t)\tanh'(u_t) - \mu_t^\top X_t\tanh'(u_t) \right] \\
        & = \frac{1}{w(1-w)}\mathbb{E} \left[  \tanh'(u_t) (\|\mu_t\|^2\tanh(u_t) + \mu_t^\top X_t) \right] 
    \end{align*}
    Hence, at the critical point we have the following condition being satisfied,
    \begin{align}
        \mathbb{E}[\tanh(u_t)] = - \frac{\mu_t^\top \mathbb{E}\left[X_t\right]}{\|\mu_t\|^2}.
    \end{align}
Now looking at the EM algorithm and in particular M-step, we see that 
\begin{align*}
    w^+ &= \mathbb{E}[r_1(X_t)] = \mathbb{E}[\frac{1}{2}(1+\tanh(u_t))]
\end{align*}
hence, plugging the convergence point of score loss into EM algorithm,
\begin{align*}
    w^+ &=  \frac{1}{2}(1-\frac{\mu_t^\top \mathbb{E}[X_t]}{\|\mu_t\|^2}), \text{ in other words EM algorithm is also converged, equivalently,} \\
    \mathbb{E}[X_t] & = \mu_t (2w-1), \text{ which is exactly the first order moment matching condition}.
\end{align*}
In other words, convergence of score estimation loss corresponds EM algorithm convergence that satisfies the first order moment equality, since the means and covariances are known this corresponds to the global maximum of the likelihood. 

At the convergence point we can measure the mean squared error of estimating the mixing weight from samples. In particular, let us define $\hat{w}= \frac{1}{2}(1-\frac{\mu_t^\top \hat{X}_t}{\|\mu_t\|^2})$ as the sample mixing weight, where $\hat{X}_t$ is the sample mean at time $t$. Then the $L_2$ error can be found to be:

\begin{align*}
        \mathbb{E}\|w-\hat{w}\|^2 &= \frac{1}{4}\|\frac{\mu_t^\top}{\|\mu_t\|^2} (\mathbb{E}[X_t]-\hat{X}_t)\|^2 \\
        & \leq \frac{1}{4\|\mu_t\|^2}\| \mathbb{E}[X_t]-\hat{X}_t\|^2 \\
        & \leq \frac{1}{4\|\mu_t\|^2} \frac{tr(\Sigma)}{n}, \text{ where } \Sigma \text{ is true overall covariance vector } \\
        & = \frac{1}{4\|\mu_t\|^2} \frac{d+4w(1-w)\|\mu_t\|^2}{n} \\
        & = \frac{w(1-w)}{n}+ \frac{d}{4 \|\mu_t\|^2n},
\end{align*}
which concludes the proof.
\end{proof}

\textbf{Remark 1}. The bound on mixing‑weight error is dimension‑free because both its numerator and denominator scale with $d$, unlike mean estimation, whose difficulty grows in high dimensions. Thus clients with few samples can reliably learn their own weights. The bound has two components: (i) sampling term—the variance of cluster assignments; it shrinks when one mixture component dominates, (ii) separation term—inversely proportional to $\|\mu_t\|^2$; larger mean separation improves accuracy, while small separation can be offset by more data. Crucially, the bound depends only on the norm $\|\mu_t\|$—not on how accurately the mean is estimated—because correct assignment, not mean precision, drives weight estimation. Overall generation quality will, however, reflect errors in both $\hat\mu$ and$\hat w$. With the mixing‑weight bound in hand, we next assess score error: we measure the mean‑squared gap between the true score and that produced by the converged estimators, where the global backbone learns $\hat\mu_t$ via DDPM gradient descent \cite{shah2023learning} and each client refines its own $\hat w$ using our personalized procedure.

\begin{theorem} \label{thm:overall}
    Let $L_{est}(\hat{\mu},\hat{w})$ be the score matching loss between the true score function and estimated score function where the $\hat \mu_t$ is estimated globally using the procedure in \cite{shah2023learning}, and mixing weights are estimated individually by doing gradient steps in the personalized DDPM loss. Assuming $B$ is a constant such that $\|\mu_t\|^2\leq B$. Then, the loss has the following error,
    \begin{align}
        L_{est}(\hat{\mu},\hat{w})=\mathcal{O}\Big(\frac{d^6B^8}{mn(1-e^{-2t})^2}\Big) + \mathcal{O}\Big(\frac{1}{n}\Big)
    \end{align}
\end{theorem}

\begin{proof}[Proof of Theorem \ref{thm:overall}.] Let us define $u:={X_t^\top \mu_t - \frac{1}{2}\log((1-w)/(w))}, \text{ and similarly for $\hat u$}$ with estimated parameters, and $b(w)=\frac{1}{2}\log((1-w)/w)$, we also assume there is a constant such that $\|\mu_t\|^2\leq B$. Then we have,
\begin{align*}
    L_{est}(\hat{\mu},\hat{w}) &= \mathbb{E}_{X_o\sim \mathcal{D}_i,Z_t}\Big\|\tanh({X_t^\top \mu_t - \frac{1}{2}\log((1-w)/(w))})\mu_t \\
    & \quad -\tanh({X_t^\top \hat\mu_t - \frac{1}{2}\log((1-\hat w)/(\hat w))})\hat \mu_t\Big\|^2 \\
    &= \mathbb{E}\|\tanh{(u)}\mu_t-\tanh (\hat u) \mu_t + \tanh (\hat u) \mu_t - \tanh (\hat u) \hat \mu_t  \|^2,  \\
    & \leq \|\mu_t\|^2 (\tanh{u}-\tanh{\hat u})^2 + |\tanh{\hat u}|\|\mu_t - \hat \mu_t\|^2 \\
    & \leq \mathbb{E} [\|\mu_t\|^2 2({u}-{\hat u})^2 + \|\mu_t - \hat \mu_t\|^2] \\
    & = \mathbb{E}[\|\mu_t\|^2 2(X_t^\top (\mu_t - \hat \mu_t)+b(w)-b(\hat w))^2 + \|\mu_t - \hat \mu_t\|^2], \\
    &\leq 4 \mathbb{E}[\|\mu_t\|^2\|X_t\|^2\|\mu_t -\hat \mu_t\|^2 + \frac{(w-\hat w)^2}{2 w_{min}^2(1-w_{min})^2}  + \|\mu_t - \hat \mu_t\|^2] \\
    &\leq 4B \mathbb{E}\|X_t\|^2 \mathbb{E} \|\hat \mu_t - \mu_t \|^2 + \mathbb{E}(w-\hat w)^2 + \mathbb{E}\|\mu_t - \hat \mu_t\|^2 \\
    & \leq 4B(d+(4w(1-w)+1)B)\mathbb{E} \|\mu-\hat \mu\|^2 + \frac{1}{2w_{min}^2(1-w_{min})^2}\frac{w(1-w)}{n},
\end{align*}

where we used Lipschitz continuity of $\mathrm{tanh}$ and $\log$ functions alongside with algebraic steps. From \cite{shah2023learning}, at the convergence point of their algorithm, we have that with probability at least 1-$2d\exp(-\frac{n\epsilon^2\beta_t^2}{d^4B^6})$ the convergence is to a ball with $\|\mu_t-\hat\mu_t\|^2=\mathcal{O}(\epsilon)$. Now we convert this high probability bound to a expectation bound. In particular the non-vanishing $\epsilon$ term is introduced in Lemma E.7 of \cite{shah2023learning}. There we have the following high probability bound for bounding the difference between population and empirical gradient,

\begin{align*}
    Pr[\Big \|\nabla_{\hat \mu_t}\frac{1}{n} \sum_{j=1}^n L_t(s_{\hat \mu_t}(x_{j,t})) - L_t(s_{\hat \mu_t})  \Big\|>\epsilon] \leq 2d \exp\Big( - \frac{n \epsilon^2 \beta_t^2}{d^2B^6}\Big)
\end{align*}

equivalently for some constant $C>0$, at the end of convergence of algorithm (where exponential decaying contraction term has vanished), the result in \cite{shah2023learning} yields, 

\begin{align*}
    Pr[ \|\mu_t-\hat \mu_t \|>C\epsilon] \leq 2d \exp\Big( - \frac{n \epsilon^2 \beta_t^2}{d^4B^6}\Big)
\end{align*}


For conversion, we use the tail definition of expectation, $\mathbb{E}[Y^2] = \int_0^\infty P(Y^2>y) dy = \int_0^\infty P(Y>\sqrt{y}) dy$ which yields,
\begin{align*}
    \mathbb{E}\|\mu_t-\hat \mu_t \|^2 = \int_0^\infty 2Cd \exp\Big( - \frac{n y \beta_t^2}{d^4B^6}\Big) dy \leq C\frac{d^5B^6}{mn\beta_t^2}, 
\end{align*}

where we used the fact that $\int_0^\infty \exp(-Cx)dx=\frac{1}{C}$ As a result, the overall bound is,

\begin{align}
    L(\hat{\mu},\hat{w}) = \mathcal{O}\Big(\frac{d^6B^8}{mn(1-e^{-2t})^2}\Big) + \mathcal{O}\Big(\frac{1}{n}\Big)
\end{align}
where the first term is due to global estimation of mean and the second term is due to the local estimation of mixing weight.

\end{proof}

\textbf{Remark 2. } Theorem~\ref{thm:overall} shows that score‑estimation error is governed by errors in both the mean and the mixing weight. Except when the client count reaches $\Theta(d^{6}B^{8})$, the mean‑estimation term dominates. Because mixing‑weight error is dimension‑free and small even with limited data, a client can accurately personalize once a good global mean is pretrained. This shows the benefit of our split.


\section{Experiments}
\textbf{Baselines.} As diffusion models require very large U-Net models to be trained, personalized FL methods with two different models, i.e. separate individual and global models, are not feasible. Hence, we compare our method to single model methods. One such baseline is the simple \textit{FedAvg+fine-tuning} for personalization, though it i s a simple approach it is accepted as one of the competitive personalization methods. Another baseline is \textit{meta-learning} approaches, specifically \cite{fallah2020personalized}, where the global model is minimizing a meta-learning based loss that allows for better personalization. Lastly, inspired by \cite{collins-icml21}, we introduce a \textit{shared representation} approach, as a competing solution to compare with, where downsample and upsample blocks are trained globally and middle block is trained individually. The goal of this such partition is to enable global learning of feature extraction (downsample blocks) and feature generation (upsample blocks) similar to global feature extraction in \cite{collins-icml21}; however, in classifier models \cite{collins-icml21} trains individual classifier heads which is not directly applicable to U-Net models. Inspired by that, we enable individual learning of middle block to introduce personalization in the latent space. We provide more details on implementation details of baselines in Appendix~\ref{app:experiments}.

\textbf{Datasets.} We randomly generate a synthetic dataset (using the true underlying score function) for synthetic experiments to estimate score function (mean and mixing weight) of GMM. For real datasets, we use \textsc{MNIST}, \textsc{CIFAR-10/100}, and \textsc{CelebA} (celebrity faces) with increasing complexity. The images have 28x28, 32x32, and 64x64 resolutions respectively. 

For pretraining on MNIST, there are 30 clients and each client has access to 400 samples from a randomly selected class; for fine-tuning on new clients, there are 30 clients each client has access to 400 samples from 3 different classes which simulates another layer of challenge, namely the new clients have samples from a different mixture dataset compared to the pretraining clients. On CIFAR-10, there are 18 clients each with 2750 samples from a randomly selected class. The new clients have samples from two randomly selected classes, but the amount of data at each client drops to 275 to simulate a data scarcity setting. On CelebA dataset there are 24 clients each with data coming from 5 different individuals, the total amount of data per client is only around 25. This setting is very challenging and acts as a stress test for our method. The new clients have data from different individuals compared to pre-training data.

\textbf{Models.} We use a conditional U-Net architecture where client id's, that are converted to vector embeddings, are used as the conditioning information. We use 14 layer, 6M parameter architecture for \textsc{MNIST} experiments, and 18 layer 13M parameter for CIFAR-10, CelebA experiments. More details can be found in Appendix. We use AdamW \cite{loshchilov2018decoupled} optimizer and provide more details on hyper-parameters in Appendix~\ref{app:experiments}.


\begin{figure}[h]
  \centering
  \captionof{table}{Pre-training performance for participating clients, and fine-tuning performance for newly joined client performance KID$\downarrow$ values of respective methods on respective datasets.}
  \vspace{-0.2cm}
  \begin{tabular}{c|ccc|ccc} \hline
    {Method} &
    \multicolumn{3}{c|}{Pre‑training} &
    \multicolumn{3}{c}{New client fine‑tuning} \\[-2pt]
    & \textsc{MNIST} & \textsc{CIFAR‑10} & \textsc{CelebA} &
      \textsc{MNIST} & \textsc{CIFAR‑10} & \textsc{CelebA} \\ \hline
    FedAvg+fine‑tuning   & 0.015 & 0.034 & 0.051 & 0.028 & 0.101 & 0.075 \\
    \spire               & \textbf{0.010} & 0.034 & \textbf{0.046} & \textbf{0.012} & \textbf{0.038} & \textbf{0.061} \\
    Shared Representation & \textbf{0.010} & \textbf{0.031} & 0.056 & 0.017 & 0.048 & 0.094 \\
    Meta learning        & 0.061 & 0.073 & 0.081 & 0.079 & 0.093 & 0.120 \\ \hline
  \end{tabular}
  \label{tab:kid}
\end{figure}

\begin{figure}
    \centering
    \includegraphics[width=1\linewidth]{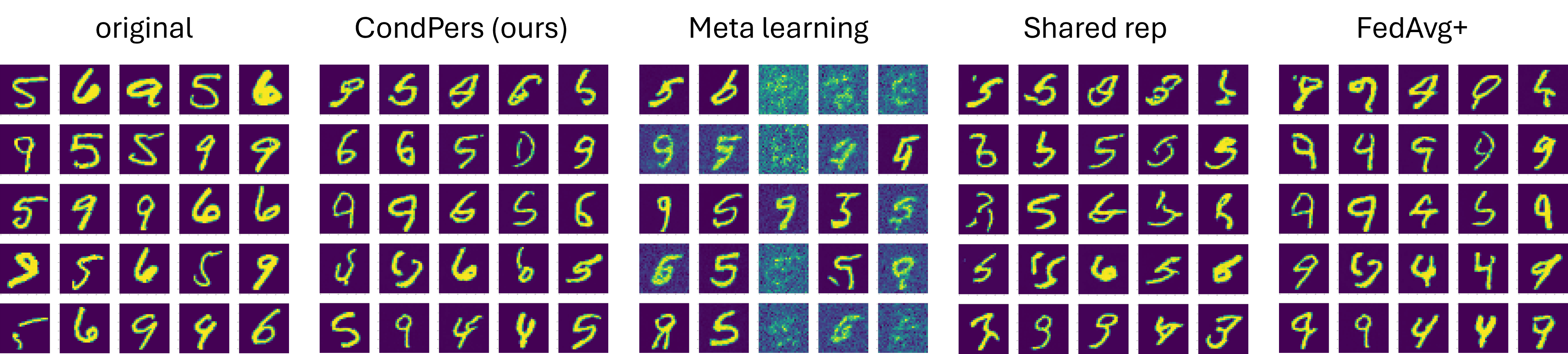}
    \caption{Grid of original \textsc{MNIST} images and generated images by competing methods for particular client, who has sampled data from classes '5,6,9; using the same random seed. Our method results in less hallucinations, more diversity while preserving structure. }
    \label{fig:mnist-grid}
\end{figure}


\subsection{Numerical Results}
\textbf{Collaborative pre-training.} The main image quality metric we use is  Kernel Inception Distance (KID) \cite{binkowski2018demystifying}, as FID \cite{heusel2017gans} is not usable for smaller datasets due to biasedness. Table \ref{tab:kid} reports KID after the collaborative pre-training phase. On \textsc{MNIST} and \textsc{CelebA}  our \spire achieves the lowest KID, improving over baselines. Particularly, the gap is significant on \textsc{CelebA}  dataset. On \textsc{CIFAR-10} \spire matches \textit{FedAvg+FT} while the \textit{Shared-Rep.} variant edges out both methods by a small margin. We attribute the \textsc{CIFAR-10} result to the fact that all clients observe only a single class during pre-training—global feature extractors alone already generalize reasonably well in that setting. Across all three data sets the meta-learning baseline struggles, confirming the practical difficulty of applying first-order MAML–style algorithms to large diffusion models.

\textbf{Generalization to new clients.} The advantage of \spire becomes far more pronounced once entirely new clients—holding data from unseen class mixtures or unseen identities—join the federation (Table \ref{tab:kid},b). After less than 10 epochs of local fine-tuning \spire yields large gains over the next best method Shared representation. Neither \textit{FedAvg+FT} nor meta-learning can close the gap; in fact, both are outperformed by \spire on every data set. These results support our central claim: \emph{embedding the client identifier as a learnable conditioning signal equips a \emph{single} diffusion model with enough capacity to adapt to novel client distributions.}

\textbf{Qualitative analysis.} Figures \ref{fig:mnist-grid} and \ref{fig:celeba-grid} (in Appendix~\ref{app:experiments}) visualize samples drawn from models trained on \textsc{CelebA}  and \textsc{MNIST}, respectively, for a particular client.  For MNIST, the client has access to images labeled as '5','6', and '9'. For the most of the images \spire is able to generate 5,6,9 images with occasional hallucination from other classes. MAML approach fails to generate quality images. Shared representation approach often misses the structure of the numbers, and FedAvg+FT does not generate diverse images implying an overfitting to the data. For \textsc{CelebA} , all methods struggle with the coloring of images due to the challenge in pretraining task. Images generated by \spire exhibit noticeably sharper facial details. Competing methods over-fit, producing more washed-out or saturated outputs. The faces generated by our method are more diverse and have better visual features e.g. shadowing.
\begin{figure}[h]
    \centering
    \begin{subfigure}[b]{0.4\linewidth}
        \centering
        \includegraphics[width=\linewidth]{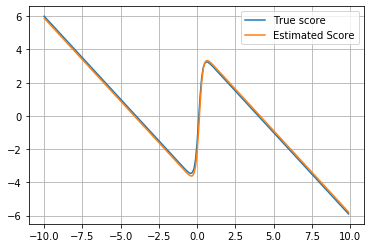}
        \label{fig:true-estimated-score}
    \end{subfigure}
    \hfill
    \begin{subfigure}[b]{0.4\linewidth}
        \centering
        \includegraphics[width=\linewidth]{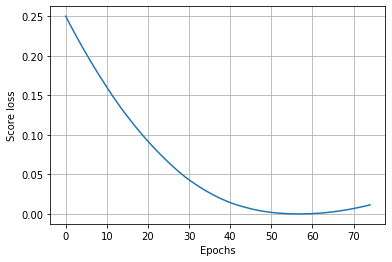}
        \label{fig:score-loss}
    \end{subfigure}
    \caption{Estimated score function (left) and score loss convergence (right) for the synthetic experiments. The true mixing weight and mean are 0.7 and 4 respectively, the estimated are 0.72 and 4.11. }
    \label{fig:synth-figures}
\end{figure}

\textbf{Synthetic experiments for GMMs. }
For synthetic experiments, we are verifying correctness of our found score function form in \eqref{eq:network}. Given the objective in \eqref{eq:score-loss}, we use Adam algorithm to update both the mean and mixing weights. The results in Figure\ref{fig:synth-figures} indicate convergence of training and success of estimating the score through the DDPM objective. See Appendix for experimental details. 

\textbf{Conditional personalization is robust to catastrophic forgetting. } A major advantage of our method, in new client fine-tuning, compared to competing methods is its robustness to overfitting, i.e. catastrophic forgetting. Since our method relies on updating only a small part of the model, previously learned information through the backbone has no risk of getting lost during the fine-tuning stage. This prevents an oversensitivity towards number of fine-tuning epochs or learning rate; which are mainly responsible for a potential forgetting. To display this notion, we compare to shared representation approach and illustrate the resulting KID values for new clients under different finetuning epochs Figure~\ref{fig:overfitting}.

\begin{figure}
  \centering
  \includegraphics[width=0.38\textwidth]{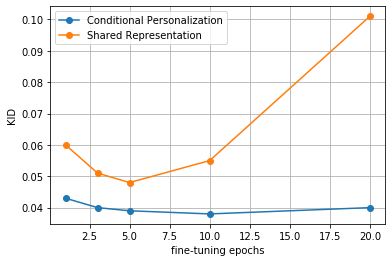}
  \caption{The number of local epochs need to be carefully fine-tuned for new clients using shared representation approach due to overfitting and underfitting risk, on the other hand conditional personalization performs better and is more robust with respect to number of local epochs..}
  \label{fig:overfitting}
\end{figure}

\textbf{Overall. }In the experiments we showed the efficacy of our proposed way of personalization. \spire is especially superior to other methods for new client fine-tuning which is a very crucial problem in FL and, in large scale ML. Due to the natural PEFT implied by our method, it is robust to choices such as learning rate and number of iterations. It does not carry risks such as catastrophic forgetting. 
\section{Conclusion}
We introduced \spire, the first framework that treats personalized diffusion in federated learning as a \emph{conditional generation} problem.  By decomposing the model into a global backbone and lightweight client embeddings, \spire enables parameter‑efficient, communication‑free adaptation on devices that hold only a handful of samples.  Our theory establishes a link between gradient descent on the DDPM objective and Gaussian‑mixture models, yielding dimension‑free error bounds for the client‑specific parameters.  Empirically, \spire consistently matches or surpasses strong baselines during collaborative pre‑training and delivers superior performance for new‑client personalization, resulting quantitatively and qualitatively better generation. Future directions include (i) extending the embedding‑based conditioning to multi‑modal and text‑to‑image diffusion, (ii) integrating differential‑privacy guarantees, and (iii) exploring hierarchical embeddings that capture groups of clients. Main limitation of our work is that it requires changing pre-trainig recipe to be done in conditioned on clients, hence, another future work that we would like to explore is integrating \spire into large pre-trained diffusion models such as StableDiffusion \cite{rombach2022highresolutionimagesynthesislatent}. We hope our results spur further research on principled, efficient personalization for large generative models in federated settings.

\newpage
\bibliographystyle{plain}
\bibliography{bibliography}

\begin{thebibliography}{10}

\bibitem{acar2021debiasing}
Durmus Alp~Emre Acar, Yue Zhao, Ruizhao Zhu, Ramon Matas, Matthew Mattina, Paul Whatmough, and Venkatesh Saligrama.
\newblock Debiasing model updates for improving personalized federated training.
\newblock In {\em International Conference on Machine Learning}, pages 21--31. PMLR, 2021.

\bibitem{binkowski2018demystifying}
Mikołaj Bińkowski, Dougal~J. Sutherland, Michael Arbel, and Arthur Gretton.
\newblock Demystifying {MMD} {GAN}s.
\newblock In {\em International Conference on Learning Representations}, 2018.

\bibitem{chen2022selfaware}
Huili Chen, Jie Ding, Eric~William Tramel, Shuang Wu, Anit~Kumar Sahu, Salman Avestimehr, and Tao Zhang.
\newblock Self-aware personalized federated learning.
\newblock In Alice~H. Oh, Alekh Agarwal, Danielle Belgrave, and Kyunghyun Cho, editors, {\em Advances in Neural Information Processing Systems}, 2022.

\bibitem{chen2023sampling}
Sitan Chen, Sinho Chewi, Jerry Li, Yuanzhi Li, Adil Salim, and Anru Zhang.
\newblock Sampling is as easy as learning the score: theory for diffusion models with minimal data assumptions.
\newblock In {\em The Eleventh International Conference on Learning Representations}, 2023.

\bibitem{chen2024learninggeneralgaussianmixtures}
Sitan Chen, Vasilis Kontonis, and Kulin Shah.
\newblock Learning general gaussian mixtures with efficient score matching, 2024.

\bibitem{collins-icml21}
Liam Collins, Hamed Hassani, Aryan Mokhtari, and Sanjay Shakkottai.
\newblock Exploiting shared representations for personalized federated learning.
\newblock In Marina Meila and Tong Zhang, editors, {\em International Conference on Machine Learning (ICML)}, volume 139 of {\em Proceedings of Machine Learning Research}, pages 2089--2099. {PMLR}, 2021.

\bibitem{daskalakis17b}
Constantinos Daskalakis, Christos Tzamos, and Manolis Zampetakis.
\newblock Ten steps of em suffice for mixtures of two gaussians.
\newblock In {\em Proceedings of the 2017 Conference on Learning Theory}, volume~65 of {\em Proceedings of Machine Learning Research}, pages 704--710, 2017.

\bibitem{deng2020adaptive}
Yuyang Deng, Mohammad~Mahdi Kamani, and Mehrdad Mahdavi.
\newblock Adaptive personalized federated learning.
\newblock {\em arXiv preprint arXiv:2003.13461}, 2020.

\bibitem{dhariwal2021diffusion}
Prafulla Dhariwal and Alexander~Quinn Nichol.
\newblock Diffusion models beat {GAN}s on image synthesis.
\newblock In A.~Beygelzimer, Y.~Dauphin, P.~Liang, and J.~Wortman Vaughan, editors, {\em Advances in Neural Information Processing Systems}, 2021.

\bibitem{dinh2020personalized}
Canh~T. Dinh, Nguyen~H. Tran, and Tuan~Dung Nguyen.
\newblock Personalized federated learning with moreau envelopes.
\newblock In {\em Advances in Neural Information Processing Systems}, 2020.

\bibitem{du2021fewshot}
Simon~Shaolei Du, Wei Hu, Sham~M. Kakade, Jason~D. Lee, and Qi~Lei.
\newblock Few-shot learning via learning the representation, provably.
\newblock In {\em International Conference on Learning Representations}, 2021.

\bibitem{fallah2020personalized}
Alireza Fallah, Aryan Mokhtari, and Asuman Ozdaglar.
\newblock Personalized federated learning: A meta-learning approach.
\newblock In {\em Advances in Neural Information Processing Systems}, 2020.

\bibitem{ghosh2020efficient}
Avishek Ghosh, Jichan Chung, Dong Yin, and Kannan Ramchandran.
\newblock An efficient framework for clustered federated learning.
\newblock In {\em Advances in Neural Information Processing Systems}, 2020.

\bibitem{hanzely2020federated}
Filip Hanzely and Peter Richtárik.
\newblock Federated learning of a mixture of global and local models.
\newblock {\em arXiv preprint arXiv:2002.05516}, 2020.

\bibitem{heusel2017gans}
Martin Heusel, Hubert Ramsauer, Thomas Unterthiner, Bernhard Nessler, and Sepp Hochreiter.
\newblock Gans trained by a two time-scale update rule converge to a local nash equilibrium.
\newblock {\em Advances in neural information processing systems}, 30, 2017.

\bibitem{kohdak2019adaptive}
Mikhail Khodak, Maria-Florina~F Balcan, and Ameet~S Talwalkar.
\newblock Adaptive gradient-based meta-learning methods.
\newblock In {\em Advances in Neural Information Processing Systems}, 2019.

\bibitem{kotelevskii2022fedpop}
Nikita~Yurevich Kotelevskii, Maxime Vono, Alain Durmus, and Eric Moulines.
\newblock Fedpop: A bayesian approach for personalised federated learning.
\newblock In Alice~H. Oh, Alekh Agarwal, Danielle Belgrave, and Kyunghyun Cho, editors, {\em Advances in Neural Information Processing Systems}, 2022.

\bibitem{kwon2020emalgorithmgivessampleoptimality}
Jeongyeol Kwon and Constantine Caramanis.
\newblock The em algorithm gives sample-optimality for learning mixtures of well-separated gaussians, 2020.

\bibitem{li2020federated}
Tian Li, Anit~Kumar Sahu, Manzil Zaheer, Maziar Sanjabi, Ameet Talwalkar, and Virginia Smith.
\newblock Federated optimization in heterogeneous networks.
\newblock In {\em Proceedings of Machine Learning and Systems 2020, MLSys}, 2020.

\bibitem{lin2020ensemble}
Tao Lin, Lingjing Kong, Sebastian~U. Stich, and Martin Jaggi.
\newblock Ensemble distillation for robust model fusion in federated learning.
\newblock In {\em Advances in Neural Information Processing Systems}, 2020.

\bibitem{liu2015faceattributes}
Ziwei Liu, Ping Luo, Xiaogang Wang, and Xiaoou Tang.
\newblock Deep learning face attributes in the wild.
\newblock In {\em Proceedings of International Conference on Computer Vision (ICCV)}, December 2015.

\bibitem{loshchilov2018decoupled}
Ilya Loshchilov and Frank Hutter.
\newblock Decoupled weight decay regularization.
\newblock In {\em International Conference on Learning Representations}, 2019.

\bibitem{mansour2020approaches}
Yishay Mansour, Mehryar Mohri, Jae Ro, and Ananda~Theertha Suresh.
\newblock Three approaches for personalization with applications to federated learning.
\newblock {\em arXiv preprint arXiv:2002.10619}, 2020.

\bibitem{marfoq2021federated}
Othmane Marfoq, Giovanni Neglia, Aur{\'e}lien Bellet, Laetitia Kameni, and Richard Vidal.
\newblock Federated multi-task learning under a mixture of distributions.
\newblock {\em Advances in Neural Information Processing Systems}, 34, 2021.

\bibitem{ozkara2023}
Kaan Ozkara, Antonious Girgis, Deepesh Data, and Suhas Diggavi.
\newblock A statistical framework for personalized federated learning and estimation: Theory, algorithms, and privacy.
\newblock In {\em International Conference on Learning Representations}, 2023.

\bibitem{ozkara2025adept}
Kaan Ozkara, Bruce Huang, Ruida Zhou, and Suhas Diggavi.
\newblock {ADEPT}: Hierarchical bayes approach to personalized federated unsupervised learning.
\newblock In {\em The 28th International Conference on Artificial Intelligence and Statistics}, 2025.

\bibitem{ozkara2021quped}
Kaan Ozkara, Navjot Singh, Deepesh Data, and Suhas Diggavi.
\newblock Quped: Quantized personalization via distillation with applications to federated learning.
\newblock {\em Advances in Neural Information Processing Systems}, 34, 2021.

\bibitem{rombach2022highresolutionimagesynthesislatent}
Robin Rombach, Andreas Blattmann, Dominik Lorenz, Patrick Esser, and Björn Ommer.
\newblock High-resolution image synthesis with latent diffusion models, 2022.

\bibitem{shah2023learning}
Kulin Shah, Sitan Chen, and Adam Klivans.
\newblock Learning mixtures of gaussians using the {DDPM} objective.
\newblock In {\em Thirty-seventh Conference on Neural Information Processing Systems}, 2023.

\bibitem{smith2017federated}
Virginia Smith, Chao{-}Kai Chiang, Maziar Sanjabi, and Ameet~S. Talwalkar.
\newblock Federated multi-task learning.
\newblock In {\em Advances in Neural Information Processing Systems}, pages 4424--4434, 2017.

\bibitem{tian2020rethinking}
Yonglong Tian, Yue Wang, Dilip Krishnan, Joshua~B Tenenbaum, and Phillip Isola.
\newblock Rethinking few-shot image classification: a good embedding is all you need?
\newblock In {\em European Conference on Computer Vision}, pages 266--282. Springer, 2020.

\bibitem{vanhaesebrouck2017decentralized}
Paul Vanhaesebrouck, Aur{\'e}lien Bellet, and Marc Tommasi.
\newblock Decentralized collaborative learning of personalized models over networks.
\newblock In {\em Artificial Intelligence and Statistics}, pages 509--517. PMLR, 2017.

\bibitem{zantedeschi2020fully}
Valentina Zantedeschi, Aur{\'e}lien Bellet, and Marc Tommasi.
\newblock Fully decentralized joint learning of personalized models and collaboration graphs.
\newblock In {\em International Conference on Artificial Intelligence and Statistics}, pages 864--874. PMLR, 2020.

\bibitem{zhang2021personalized}
Michael Zhang, Karan Sapra, Sanja Fidler, Serena Yeung, and Jose~M. Alvarez.
\newblock Personalized federated learning with first order model optimization.
\newblock In {\em International Conference on Learning Representations}, 2021.

\end{thebibliography}

\newpage
\appendix

\section{Proofs and other details for theoretical results} \label{app:proofs}
\subsection{Diffusion model setup}
\rem{We first start with the forward process. Ornstein-Uhlenbeck (OU) forward diffusion process describes the gradual transformation from a true data sample to noise through a SDE, it is defined as
\begin{align} \label{eq:forward}
    \mathrm{d} X_t = -\delta_t X_t \, \mathrm{d}t + \sqrt{2 \delta_t} \, \mathrm{d} {W}_t, \quad X_0 \sim q_0(x),
\end{align}
here $q_0(x)$ denotes the true underlying distribution, similarly $q_t(x)$ will be used to denote a latent distribution, $\delta_t$ is a drift coefficient, commonly assumed a constant e.g. 1. ${W}_t$ is the standard Brownian motion. $X_t \sim q_t$ is a latent variable. Equivalently, it can be shown for some $\beta_t:\lim_{t \rightarrow\infty}\beta_t = 1, \beta_0=0$,
\begin{align*}\label{eq:equ-fwp}
    X_t = \sqrt{1-\beta_t^2}X_0 + \beta_t Z_t, \text{with } X_0 \sim q_0 \text{ and } Z_t \sim \mathcal{N}(0,\mathrm{I}).
\end{align*}
Corresponding to the forward process, is the backward SDE process, which is defined to formalize recovering a data sample from true distribution, given a corrupted sample and score function, 
\begin{align*}
    \mathrm{d} \overleftarrow{X}_{t} = 
\left[ 
\delta_{T-t} \overleftarrow{X}_{t} + 2 \delta_{T-t} \nabla \log q_{T-t}(\overleftarrow{X}_{t}) 
\right] \mathrm{d}t 
+ \sqrt{2 \delta_{T-t}} \, \mathrm{d} {W}_t, 
\quad \overleftarrow{X}_{0} \sim q_T(\cdot).
\end{align*}
In the backward process, we use $\overleftarrow{X}_{0}$ to denote a sample from pure noise $q_T$. $\nabla \log q_{T-t}(\overleftarrow{X}_{t})$ is the true score function at time $T-t$ (or at time $t$ in terms of the forward process). Of course the true score function is unknown, hence we need to estimate it from the data using a fixed pure noise distribution. To that end, one constructs the score matching optimization criterion,
\begin{align}
\min_{\theta} \int_\tau^T 
\mathbb{E}_{X_t \sim q_t(\cdot)} 
\left[
\left\|
s_{\theta}(x_t, t) - \nabla \log q_t( X_0)
\right\|^2
\right] 
\, \mathrm{d}t,
\end{align}
where $\theta$ is a neural network that is assumed to parameterize the score function. Equivalent to score matching is the DDPM view, which yields,
\begin{align}
\min_{\theta} \int_\tau^T 
\mathbb{E}_{X_0 \sim q_0(\cdot), Z_t} 
\left[
\left\|
s_{\theta}(X_t, t) + \frac{Z_t}{\beta_t}
\right\|^2
\right] 
\, \mathrm{d}t.
\end{align}
}

\textbf{EM updates.} EM algorithm is the most common approach to find parameters of GMMs. E-step is consists of finding the soft assignments or so called responsibilities, and M-step corresponds to updating the parameters. As can be seen in \cite{kwon2020emalgorithmgivessampleoptimality}:
\begin{align*}
\text{(E-step)}:&\quad r_i(X) = \frac{w_i \exp(-\|X - \mu_i\|^2/2)}{\sum_{j=1}^k w_j \exp(-\|X - \mu_j\|^2/2)} \substack, \text{ where subscripts denote cluster assignments}\\
\text{ in our case:} & \quad r_1(X) =  \frac{w \exp(-\|X - \mu\|^2/2)}{ w \exp(-\|X - \mu\|^2/2)+(1-w) \exp(-\|X + \mu\|^2/2)}, r_2(X) = 1- r_1(X) \\[6pt]
\text{(M-step)}:&\quad 
w^+ = \mathbb{E}_X[r_1(X)],\quad 
\mu^+ = \mathbb{E}_X[r_1(X) X]/\mathbb{E}_X[r_1(X)] \text{ where $+$ denotes the next iteration}.
\end{align*}

\subsection{Additional Proofs in Section ~\ref{sec:theory}}
\begin{lemma}[Solution of the OU forward SDE]\label{lem:ou_solution}
Let \((W_t)_{t\ge 0}\) be a standard $d$-dimensional Brownian motion independent of
the initialization $X_0 \sim q_0$.
Consider the Ornstein–Uhlenbeck (OU) forward (noising) process
\begin{equation}\label{eq:ou_sde}
    dX_t \;=\; -\,X_t\,dt \;+\; \sqrt{2}\,dW_t,
    \qquad X_{0}\sim q_{0}.
\end{equation}
Then for every $t\ge 0$ the marginal distribution of $X_t$ is
\begin{equation}\label{eq:ou_marginal}
    X_t \;=\; \mathrm e^{-t}\,X_0
          \;+\;
          \sqrt{1-\mathrm e^{-2t}}\;Z_t,
    \qquad Z_t \sim \mathcal N\!\bigl(0,\mathrm I\bigr),
\end{equation}
with $Z_t$ independent of $X_0$.
\end{lemma}

\begin{proof}
Multiply \eqref{eq:ou_sde} by the integrating factor $\mathrm e^{t}$:
\[
    \mathrm e^{t}\,dX_t + \mathrm e^{t}\,X_t\,dt
    \;=\;
    \sqrt{2}\,\mathrm e^{t}\,dW_t
    \;\Longrightarrow\;
    d\!\bigl(\mathrm e^{t}X_t\bigr)
    = \sqrt{2}\,\mathrm e^{t}\,dW_t.
\]
Integrating from $0$ to $t$ and dividing by $\mathrm e^{t}$ yields
\[
    X_t
    \;=\;
    \mathrm e^{-t}X_0
    \;+\;
    \sqrt{2}\,\mathrm e^{-t}\!\int_{0}^{t}\!\mathrm e^{s}\,dW_s.
\]
Because the Itô integral is Gaussian with mean~$0$, its covariance is
\[
    \operatorname{Cov}\!\Bigl(\sqrt{2}\,\mathrm e^{-t}\!\int_{0}^{t}\!\mathrm e^{s}\,dW_s\Bigr)
    = 2\,\mathrm e^{-2t}\int_{0}^{t}\!\mathrm e^{2s}\,ds
    = 1-\mathrm e^{-2t}.
\]
Thus the stochastic integral can be written as
\(
    \sqrt{1-\mathrm e^{-2t}}\;Z_t
\)
with $Z_t \sim \mathcal N(0,\mathrm I)$, independent of $X_0$.
Substituting back gives \eqref{eq:ou_marginal}, completing the proof.
\end{proof}

\begin{lemma}[Two–component mixture is preserved by the OU forward process]
\label{lem:ou_mixture_two}
Let the initial distribution be the symmetric two–component Gaussian mixture  
\[
q_{0}(x)=w\,\mathcal{N}\!\bigl(x;\mu,\mathbf{I}\bigr)\;+\;(1-w)\,\mathcal{N}\!\bigl(x;-\mu,\mathbf{I}\bigr),
\qquad 0<w<1,\;\;\mu\in\mathbb{R}^{d}.
\]
Consider the Ornstein–Uhlenbeck (OU) forward SDE
\[
X_t \;=\; e^{-t} X_{0} \;+\;\sqrt{1-e^{-2t}}\;Z_t,
\qquad\text{with}\; Z_t\sim\mathcal{N}(0,\mathbf{I})\;\text{independent of }X_{0}.
\]
Then, for every $t\ge 0$, the law of $X_t$ remains a two–component mixture
\[
q_{t}(x)=w\,\mathcal{N}\!\bigl(x;\mu_t,\mathbf{I}\bigr)\;+\;(1-w)\,\mathcal{N}\!\bigl(x;-\mu_t,\mathbf{I}\bigr),
\qquad\text{where}\;\;\mu_t \;:=\;e^{-t}\mu .
\]
\end{lemma}

\begin{proof}
Condition on the component of the mixture that generated $X_0$.

\paragraph{Component ``$+\mu$''.}
If $X_{0}\sim\mathcal{N}(\mu,\mathbf{I})$, then
\[
X_t \;=\; e^{-t}X_{0}+\sqrt{1-e^{-2t}}\,Z_t
        \;\sim\; \mathcal{N}\!\bigl(e^{-t}\mu,\;e^{-2t}\mathbf{I}+(1-e^{-2t})\mathbf{I}\bigr)
        \;=\;\mathcal{N}\!\bigl(e^{-t}\mu,\mathbf{I}\bigr).
\]

\paragraph{Component ``$-\mu$''.}
Analogously, if $X_{0}\sim\mathcal{N}(-\mu,\mathbf{I})$, we obtain 
\[
X_t\;\sim\;\mathcal{N}\!\bigl(-e^{-t}\mu,\mathbf{I}\bigr).
\]

\paragraph{Mixture.}
Because the two cases occur with probabilities $w$ and $1-w$, the unconditional
distribution of $X_{t}$ is the weighted sum of the two Gaussians derived above,
namely
\[
q_{t}(x)=w\,\mathcal{N}\!\bigl(x;e^{-t}\mu,\mathbf{I}\bigr)+(1-w)\,\mathcal{N}\!\bigl(x;-e^{-t}\mu,\mathbf{I}\bigr),
\]
which coincides with the claimed form after defining $\mu_t=e^{-t}\mu$.
\end{proof}

\begin{proof}[Proof of Lemma \ref{lem:score}. ]
Recall that\; $q_t$ is the symmetric two–component mixture
\[
q_t(x)=w\,\mathcal N\!\bigl(x;\mu_t,\mathrm I\bigr)+(1-w)\,\mathcal N\!\bigl(x;-\mu_t,\mathrm I\bigr),
\qquad  
0<w<1,\;\;\;\mu_t=e^{-t}\mu .
\]

Because $q_t(x)=\sum_{k=1}^{2}\pi_k p_k(x)$ with 
$p_1(x)=\mathcal N(x;\mu_t,\mathrm I),\;\;p_2(x)=\mathcal N(x;-\mu_t,\mathrm I)$ and
weights $\pi_1=w,\;\pi_2=1-w$, the score can be written as  
\[
\nabla_x\log q_t(x)
      =\frac{\displaystyle\sum_{k=1}^{2}\pi_k
                 \,p_k(x)\,\nabla_x\log p_k(x)}
                {\displaystyle\sum_{k=1}^{2}\pi_k p_k(x)} .
\]
For a Gaussian with unit covariance,  
$\nabla_x\log \mathcal N(x;m,\mathrm I)=-(x-m)$.  Hence 
\[
\nabla_x\log q_t(x)
       =-\Bigl[r_1(x)\bigl(x-\mu_t\bigr)+r_2(x)\bigl(x+\mu_t\bigr)\Bigr],
\]
where 
\[
r_1(x)\;=\;\frac{w\,p_1(x)}{w\,p_1(x)+(1-w)p_2(x)}, 
\qquad   
r_2(x)=1-r_1(x)
\]
are the posterior (“belief’’) probabilities of the two components.
A short algebraic simplification gives  
\begin{equation}\label{eq:score_basic}
\nabla_x\log q_t(x)\;=\;\bigl(2r_1(x)-1\bigr)\,\mu_t \;-\;x .
\end{equation}

Writing the Gaussian densities explicitly:
\[
p_1(x)=\frac{1}{(2\pi)^{d/2}}\exp\!\Bigl[-\tfrac12\|x-\mu_t\|^2\Bigr],
\quad
p_2(x)=\frac{1}{(2\pi)^{d/2}}\exp\!\Bigl[-\tfrac12\|x+\mu_t\|^2\Bigr].
\]
Then
\[
\frac{r_1(x)}{1-r_1(x)}
 =\frac{w\,p_1(x)}{(1-w)\,p_2(x)}
 =\frac{w}{1-w}\,
   \exp\!\Bigl[-\tfrac12(\|x-\mu_t\|^2-\|x+\mu_t\|^2)\Bigr].
\]
Because 
$\|x-\mu_t\|^2-\|x+\mu_t\|^2=-4\,\mu_t^\top x$, we obtain
\[
\frac{r_1(x)}{1-r_1(x)}
   =\exp\!\Bigl[\,2\mu_t^\top x\;+\;\log\!\frac{w}{1-w}\Bigr].
\]
Setting 
$
z(x):=2\mu_t^\top x\;+\;\log\!\frac{w}{1-w}
$
and recalling that 
$r_1=\sigma(z)$ with the logistic function $\sigma(z)=1/(1+e^{-z})$, we have  
$2r_1-1=\tanh\!\bigl(\tfrac12 z\bigr)$.  Therefore
\[
2r_1(x)-1
  \;=\;\tanh\!\Bigl(\mu_t^\top x+\tfrac12\log\frac{w}{1-w}\Bigr)
  \;=\;\tanh\!\Bigl(\mu_t^\top x-\tfrac12\log\frac{1-w}{w}\Bigr).
\]

Substituting the expression for $2r_1(x)-1$ back into 
\eqref{eq:score_basic} yields
\[
\nabla_x\log q_t(x)
   =\tanh\!\Bigl(\mu_t^\top x-\tfrac12\log\frac{w}{1-w}\Bigr)\mu_t
     \;-\;x,
\]
which is exactly the claimed form in Lemma~\ref{lem:score}.
\end{proof}

\section{Additional Details and Results on Experiments} \label{app:experiments}

\textbf{Experiment details. } We use Adam optimizer with constant learning rate following the original DDPM paper. For \textsc{CelebA, CIFAR-10} we use learning rate of $1e-4$, and for \textsc{MNIST} we use $1e-3$. During training, each round of local training corresponds to 1 epoch. For \textsc{CelebA}, we traing for a total of 800 epochs with 50 local iterations, for \textsc{CIFAR-10} we train for 400 epochs with 100 local iterations, and for \textsc{MNIST} we train for 200 epochs using 20 local iterations. For fine-tuning, for competing methods $\times 0.1$ of original learning rate gives the best results, for our method we use $0.01$ as the learning rate. 

\textbf{Model architecture. }For every dataset we employ the same high-level design: a identity-embedding layer maps each identity token 
\(y\in\{0,\dots,m-1\}\) to a learnable vector of dimension \(d_{\mathrm{emb}}\); this vector is injected into an \texttt{UNet2DModel} backbone through modulating activations.  Each backbone uses two
residual blocks per scale, while its depth and channel widths are dataset-specific:

\begin{itemize}[leftmargin=*]
\item \textbf{MNIST} (\(28\times 28\), 1 channel): three scales with feature widths
  \((32, 64, 64)\).
  The encoder follows
  \texttt{DownBlock2D} \(\rightarrow\)
  \texttt{AttnDownBlock2D} \(\rightarrow\)
  \texttt{AttnDownBlock2D}; the decoder mirrors this with two
  \texttt{AttnUpBlock2D}s and a final \texttt{UpBlock2D}.
  This lightweight configuration suffices for the low-resolution digit images.

\item \textbf{CIFAR-10} (\(32\times 32\), 3 channels): four scales with widths
  \((64, 128, 128, 128)\).
  The down path is
  \texttt{DownBlock2D} \(\rightarrow\)
  \texttt{AttnDownBlock2D} \(\rightarrow\)
  \texttt{AttnDownBlock2D} \(\rightarrow\)
  \texttt{DownBlock2D}; the up path is
  \texttt{UpBlock2D} \(\rightarrow\)
  \texttt{AttnUpBlock2D} \(\rightarrow\)
  \texttt{AttnUpBlock2D} \(\rightarrow\)
  \texttt{UpBlock2D}.
  Both width and depth are doubled relative to MNIST to handle the richer natural-image content.

\item \textbf{CelebA} (\(64\times 64\), 3 channels): we reuse the CIFAR backbone.
\end{itemize}

Across all variants, the parameter count rises from MNIST to
CIFAR/CelebA, yet the conditioning interface is unchanged. For our implementation we directly use label conditioning interface of UNet architecture. 

\textbf{ \textsc{CelebA} qualitative results. } In Figure~\ref{fig:celeba-grid} we can observe generated images for the modesl trained on CelebA dataset.

\begin{figure}[h]
    \centering
    \includegraphics[width=1\linewidth]{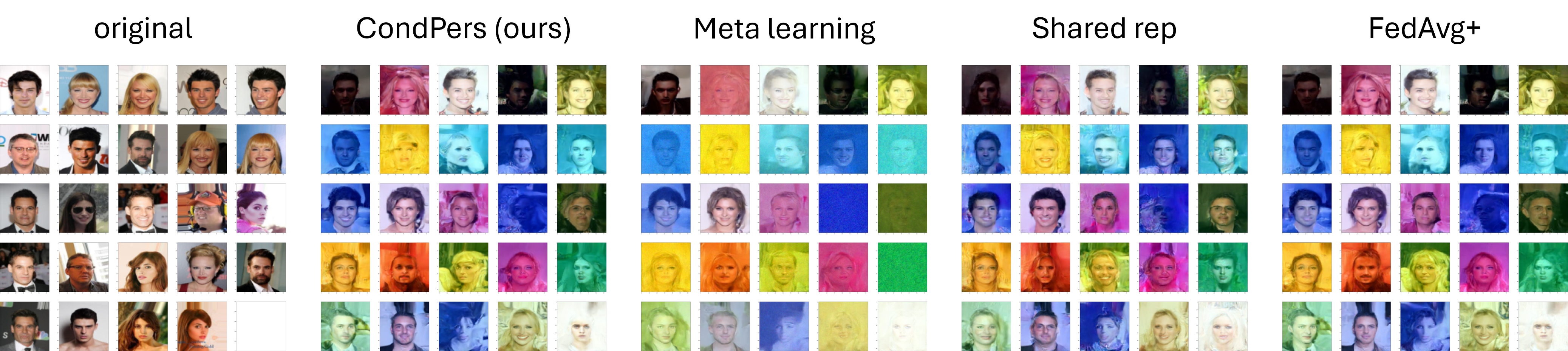}
    \caption{\rem{Grid of original \textsc{CelebA} images and generated images by competing methods for particular client, who has sampled data from 5 different individuals; using the same random seed \kaan{maybe to appendix}. }}
    \label{fig:celeba-grid}
\end{figure}


\textbf{Shared representation approach architecture. } Inspired by \cite{collins-icml21}, in our framework \spire, we introduce a novel \textit{shared representation} approach as a strong competing baseline. Concretely, each client’s U-Net is partitioned into three semantic parts: (i) the \emph{down-sampling path}, responsible for hierarchical feature extraction; (ii) the \emph{bottleneck (middle) block}, where the highest-level, most compressed latent representation is processed; and (iii) the \emph{up-sampling path}, which decodes latents back into the pixel space. During federated training, the down- and up-sampling paths are synchronized across clients at every communication round in \spire—mirroring the global feature extractor in \cite{collins-icml21} so that low and high-level visual features are learned from the entire population’s data as can be seen in Figure \ref{fig:sharedrep-arch}. In contrast, the bottleneck block is kept \emph{client-specific} and never uploaded to the server. This design has several advantages:

\begin{enumerate}
\item \textbf{Personalization without classification head swapping.} Unlike classifiers, diffusion U-Nets do not have a lightweight prediction head that can be individualized. By individualizing the bottleneck, we inject client-specific inductive bias at the deepest latent level, where semantic factors are most disentangled, enabling fine-grained customization of generation quality and style.
\item \textbf{Better fine-tuning performance.} In \spire, global synchronization of the encoder and decoder paths anchors each client’s individual bottleneck in a common feature space, mitigating overfitting to narrow data distributions and preserving cross-client coherence—a problem that pure fine-tuning baselines (e.g., FedAvg+FT) often suffer from. Note that, as we see in Section 5, conditional personalization within \spire still results in better fine-tuning performance.
\end{enumerate}

\textbf{Training hardware. } For training we use a GPU server of 6 NVIDIA RTX2080 GPUs. The longest training is done on \textsc{CelebA}, which lasts around 60 hours.

\begin{figure}[h]
    \centering
    \includegraphics[width=0.5\linewidth]{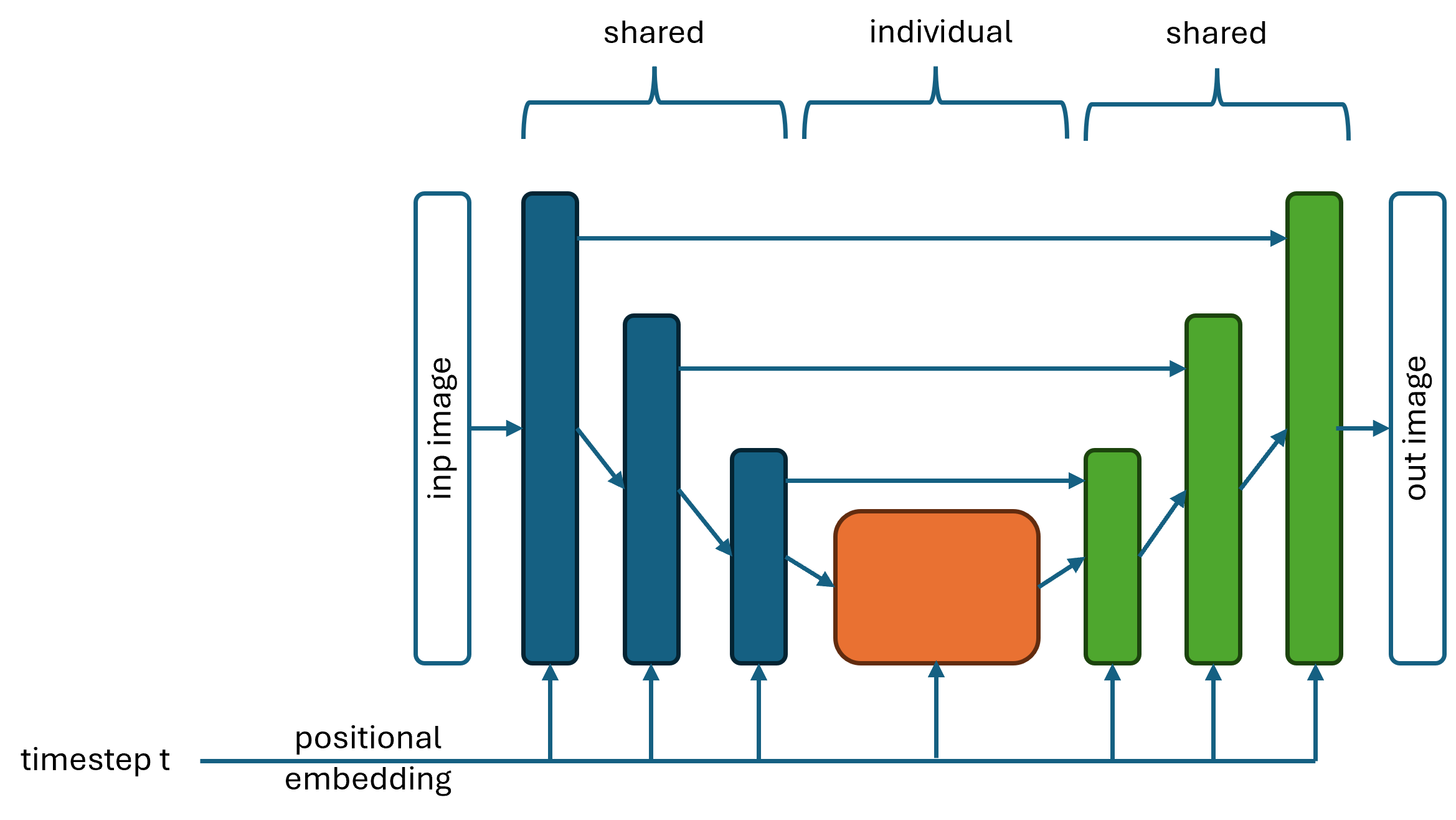}
    \caption{Shared representation partitioning. } 
    \label{fig:sharedrep-arch}
\end{figure}

\newpage

\end{document}